\documentclass[conference] {IEEEtran}
\usepackage{graphicx} 
\usepackage{hyperref}
\usepackage{url}
\usepackage{cite}
\usepackage{mathtools}
\usepackage{amsmath}
\usepackage{amsthm}
\usepackage{amssymb}
\usepackage{tikz}

\usepackage{xfrac}
\usepackage{enumerate}

\newtheorem{theorem}{Theorem}

\theoremstyle{definition}
\newtheorem{definition}{Definition}
\theoremstyle{remark}
\newtheorem{remark}{Remark}

\title{Data Augmentation and Regularization for Learning Group Equivariance}
\author{\IEEEauthorblockN{Oskar Nordenfors, Axel Flinth}
\IEEEauthorblockA{Umeå University, Department of Mathematics and Mathematical Statistics}}
\date{January 2025}

\begin{document}

\maketitle
\begin{abstract}
In many machine learning tasks, known symmetries can be used as an inductive bias to improve model performance. In this paper, we consider learning group equivariance through training with data augmentation. We summarize results from a previous paper of our own, and extend the results to show that equivariance of the trained model can be achieved through training on augmented data in tandem with regularization.
\end{abstract}
\section{Introduction}
In certain machine learning tasks, symmetries of the data distribution are known a priori. For example, when estimating energy levels of a molecule, they should not change through a rotation. For such tasks, \emph{group equivariant models} have achieved state-of-the-art performance. A prominent example is the AlphaFold 2 model\cite{jumper2021highly}. Group equivariant models have been the focus of much research, and there are many ways to make models equivariant.

One line of research, starting with \cite{cohen2016group, cohen2017steer}, considers restricting the linear layers of neural networks to ensure that they are equivariant regardless of the values of the learned weights. That is, to achieve equivariance in the model, one makes the model equivariant by design, as opposed to, say, learning the equivariance. One benefit of equivariance-by-design is that there is a clear guarantee that the model will be exactly equivariant throughout training.

Another way to approach the problem is to consider learning equivariance from data. Intuitively, one way to do this is to use data augmentation. This entails supplementing the data through transformations of the inputs and outputs in order to symmetrize the training data distribution. Data augmentation was put into a group-theoretic context in \cite{chen2020group}, which has led to a great deal of research. Several works have since dealt with the question whether data augmentation (provably) induces equivariance. However, previous studies have been confined to simpler settings, such as linear models~\cite{lyle2020benefits,elesedy2021provably} as well as linear neural networks~\cite{lawrence2021implicit,chen2023implicit}, that is neural networks without activation functions. However, results for bona fide neural networks are scarce. Despite this, data augmentation is widely used and to great effect, for example in the recent AlphaFold 3 model \cite{abramson2024accurate}. This motivates searching for theoretical results also for 'real' neural networks.

In \cite{nordenfors2024optimizationdynamicsequivariantaugmented}, the authors of this paper developed a framework for studying the effects of data augmentation also for more general neural networks. The main results were that, under some geometrical conditions on the nominal neural network architecture, (i) the set of equivariant architectures $\mathcal{E}$ is an invariant subset for the dynamics of gradient descent with augmented data and (ii) the set of stationary points on $\mathcal{E}$ are the same for augmented and equivariant dynamics. $\mathcal{E}$ could however be unstable for the augmented dynamics. 

In this paper, we will present this framework and these results of \cite{nordenfors2024optimizationdynamicsequivariantaugmented}, and also study the effect of a simple regularization procedure. We will show that using this strategy in tandem with data augmentation will make the set of equivariant architectures $\mathcal{E}$ for attractor for the training dynamics. We will also provide a small numerical experiment confirming our results.

Although we will not go into this further in this article, let us also mention another recent development related to so-called \emph{ensembles}. In that setting, it is possible to show  group equivariance through data augmentation \cite{gerken2024emergentequivariancedeepensembles, maass2024symmetriesoverparametrizedneuralnetworks, nordenfors2024ensemblesprovablylearnequivariance} in a global sense. These results rely heavily on taking averages over many individually trained networks, which can be expensive. This paper treats instead individual networks.




\section{Background}
We will here present the framework and main results of \cite{nordenfors2024optimizationdynamicsequivariantaugmented}. Let us begin by defining what we mean by group representation and group equivariance. These are standard definitions and can be found in, for example, \cite[p. 3]{fulton2004representation}.
\begin{definition}
    A \textit{representation} (\textit{rep}) of a group $G$ on a vector space $V$ is a group homomorphism $\rho:G\to\mathrm{Aut}(V)$. That is, a map that associates each group element with an invertible matrix, in a way that respects the group structure.
\end{definition}
If $G$ is compact, the reps can be assumed to be unitary w.l.o.g. In the sequel, we will assume that $G$ is compact. There is always a \textit{trivial rep} defined by $\rho(g)=\mathrm{id}$ for every $g\in G$.
\begin{definition}
    Given a group $G$ and vector spaces $U$ and $V$ with reps $\rho_U$ and $\rho_V$ respectively, a map $f:U\to V$ is called \textit{equivariant} if
    \begin{equation}\label{eq:equivariant}
        f(\rho_U(g)u)=\rho_V(g)f(u), \quad 
        \quad g\in G\, \mathrm{and}\, u\in U,
    \end{equation}
\end{definition}
If in \eqref{eq:equivariant} $\rho_V$ is the trivial rep, then $f$ is called \textit{invariant}. Given reps $\rho_U$ and $\rho_V$ on $U$ and $V$, respectively, we define a rep on the space of linear maps $U\to V$ by $\overline{\rho}(g)A=\rho_V(g)^{-1}A\rho_U(g)$. The subspace of equivariant linear maps can then be identified as $\mathrm{Hom}_G(U,V)\coloneqq\{A\in\mathrm{Hom}(U,V):\overline{\rho}(g)A=A\, \,\forall g\in G\}$. 


\subsection{A neural network framework} Let $X_0,\ldots, X_L$ be the input, intermediate, and output spaces of the network, respectively, and $\sigma_i:X_{i+1}\to X_{i+1}$, $i \in [L]$ be its non-linearities. Denoting the learnable linear layers of the network $A_i: X_i \to X_{i+1}$, we obtain a neural network by recursively defining $x_0=x$, $x_{i+1}=\sigma_{i+1}(A_ix_i)$, and finally $\Phi(x)=x_L$. Given reps $\rho_0$ and $\rho_{L}$ of the group $G$ on the input and output spaces, respectively, it is straightforward to see that if we specify a rep $\rho_i$ of $G$ on all intermediate spaces and choose both the non-linearities $\sigma_i$ and the linearities $A_i$ are chosen equivariant, $\Psi_A$ will also be. This is essentially the canonical way to construct manifestly equivariant networks\cite[p. 27]{bronstein2021geometric}. Here we refer to the network being trained in \emph{equivariant} mode. We will compare this to the corresponding \emph{nominal} architectures, that is, networks with the same non-linearities, but linear layers not necessarily restricted to the space $\mathcal{H}_G$ of (tuples of) equivariant linear maps.

While the above construction is general, it only entails fully-connected networks without bias. In \cite{nordenfors2024optimizationdynamicsequivariantaugmented}, we showed that a vast range of architectures can painlessly be incorporated into the framework by a priori restraining the linear layers $A_i$ to lie in an affine subspace $\mathcal{L}$. As a simple example, we obtain CNN:s by letting $\mathcal{L}$ be the subspace of convolution operators. We put these into equivariant mode via  constraining the layers to lie in the affine subspace $\mathcal{E}=\mathcal{H}_G \cap \mathcal{L}$.


\subsection{Training dynamics}  Let $\ell:X_L\times X_L\to \mathbb{R}$ be a loss function. Using training data and labels $(x,y)$ distributed according to a distribution $\mathcal{D}$, we define the \emph{nominal risk} $R(A)\coloneqq\mathbb{E}_{\mathcal{D}}[\ell(\Phi_A(x_0),x_L)]$. Data augmentation is performed by applying symmetry transformations $\rho_0(g), \rho_L(g)$, drawn according to a measure $\mu$. This transforms the risk function into $$R^{\mathrm{aug}}(A)\coloneqq\mathbb{E}_\mu[\mathbb{E}_{\mathcal{D}}[\ell(\Phi_A(\rho_0(g)x_0),\rho_L(g)x_L)]],$$ that we will refer to as the \emph{augmented risk}. We will always assume that $\mu$ is the so-called \emph{Haar measure} of the group, which has the property that group translations $h \mapsto gh$ are measure preserving. This is a probability measure, since we have assumed that $G$ is compact. If $G$ is finite, the Haar measure is simply the uniform measure.

Now, let us consider three approaches for training our network. Firstly, we can train our network on non-augmented data via applying gradient flow with the nominal risk. To confine the flow to the subspace $\mathcal{L}$, this gradient flow needs to be projected: Letting $\Pi_{\mathcal{L}}$ denote the orthogonal projection onto $T\mathcal{L}$ (the tangent space of the linear manifold $\mathcal{L}$), we hence consider $\Dot{A}=-\Pi_{\mathcal{L}}\nabla R(A)$. In the same manner, training on augmented data amounts to the dynamics $\Dot{A}=-\Pi_{\mathcal{L}}\nabla R^{\mathrm{aug}}(A)$. Lastly, we can train our network in equivariant mode via applying gradient flow projected to $\mathcal{E}$, that is $\Dot{A}=-\Pi_{\mathcal{E}}\nabla R(A)$, with $\Pi_{\mathcal{E}}$ the orthogonal projection onto $\mathrm{T}\mathcal{E}$. These projections do not need to be applied explicitly -- as we discuss in \cite{nordenfors2024optimizationdynamicsequivariantaugmented}, this behavior emerges simply by orthogonal parametrizations of the layers and applying gradient descent to the coefficients of those parametrizations \cite[Sec. 3]{nordenfors2024optimizationdynamicsequivariantaugmented}. 

\subsection{The 'local equivalence' of augmenting and restricting.}
The analysis we performed in \cite{nordenfors2024optimizationdynamicsequivariantaugmented} revolves around the augmented and restricted dynamics on $\mathcal{E}$. In essence, we proved three results. To arrive at them, we must assume that

(a) The loss function $\ell$ is invariant: $$\ell(\rho_{L}(g)x,\rho_{L}(g)x')=\ell(x,x'), g\in G,x,x' \in X_L.$$ 

(b) The nominal architecture satisfies the \emph{compatibility condition}, that is \emph{the orthogonal projections $\Pi_{\mathcal{L}}$ onto $T\mathcal{L}$ and $\Pi_G$ onto $\mathcal{H}_G$ should commute}.
$$\Pi_{\mathcal{L}}\Pi_G = \Pi_G \Pi_{\mathcal{L}}.$$
This is satisfied if all symmetry transformations $\rho(g)$ leave $\mathcal{L}$ invariant. This shows that the results are applicable to, for example, convolutional neural networks and rotations in image space, as long as the supports of the convolutions are rotationally symmetric.

Under these assumptions, we have the following.
\begin{theorem}
1.  $\mathcal{E}$ is an invariant set of the augmented dynamics. 

2. The set of points $S^{\mathrm{eq}}$ and $S^{\mathrm{aug}}$ \emph{in $\mathcal{E}$} that are stationary for the equivariant and augmented dynamics, respectively, agree.

\end{theorem}

Both of these results follow from the (non-trivial) fact, which we will use in the following:

\noindent
\textbf{Fact A}  $\Pi_{\mathcal{L}}\nabla R^{\mathrm{aug}}(A)= \Pi_\mathcal{E}\nabla R(A)$ for $A\in \mathcal{E}$. 

Note that the above even shows that when restricted to $\mathcal{E}$, the dynamics of the augmented and equivariant training are exactly the same. However, the result only concerns stationarity, which is, of course, different from stability. \cite{nordenfors2024optimizationdynamicsequivariantaugmented} contains a weak result about stability.

\begin{theorem}
    Points on $\mathcal{E}$ that are stable under the augmented dynamics are also stable under the equivariant dynamics, but not necessarily  the other way around.
\end{theorem}

In particular, $\mathcal{E}$ is not guaranteed to be an attractor for the augmented dynamics. In the next section, we show that we can achieve that through regularization.


\section{On the Effects of Regularization}
A standard way of regularizing neural network training is weight decay, which refers to adding a regularization term proportional to the squared norm of the weights, $\sfrac{\gamma}{2}\lVert A \rVert^2$ to the risk. This prevents the weights from getting to large. Here, we are interested in the \emph{non-equivariant} parts of the weight getting too large, which suggests using a term of the form $\sfrac{\gamma}{2}\lVert \Pi_{\mathcal{E}^\perp}A\rVert^2$.

It should be noted that it is not hard to calculate the projection onto $\mathcal{E}^\perp$ explicitly. For most important symmetry groups and architectures, orthogonal bases of $\mathcal{E}$ are known, see for example \cite{maron2018invariant, cohen2018generalGCNN}. There are even methods to calculate them numerically \cite{finzi2021practical}.

Let us now give a formal proof that given a strong enough regularization, $\mathcal{E}$ will become an attractor. We will make the following simplifying assumptions.
\begin{enumerate}[(i)]
    \item The risk is bounded below by zero.
    \item The second derivative of the augmented risk is bounded below in the following sense: There exists a (\emph{not necessarily positive}) constant $\sigma$ such that 
    \begin{align*}
        \langle{(R^{\mathrm{aug}})''(A)Y,Y\rangle} \geq \sigma \Vert Y \Vert^2 , \quad A \in \mathcal{E}, Y \in T\mathcal{E}^\perp.
    \end{align*}
    \item The third derivative of the augmented risk is uniformly bounded.
\end{enumerate}

Before we come to the main result, let us present another fact derived in \cite[proof of Prop. 3.14]{nordenfors2024optimizationdynamicsequivariantaugmented} that will be used in the proof.

\noindent
\textbf{Fact B} If $Y \in T\mathcal{E}^\perp$ and $A\in \mathcal{E}$, $\Pi_{\mathcal{L}} (R^{\mathrm{aug}})''(A)Y \in T\mathcal{E}^\perp$.
Let us also introduce the notation $\Pi_{\mathcal{E}^\perp}$ for the orthogonal projection onto $\mathrm{T}\mathcal{E}^\perp$.

We can now prove the main result, that the equivariant subspace $\mathcal{E}$ becomes an attractor for training with data augmentation by penalizing non-equivariance through a regularization term in the augmented loss.
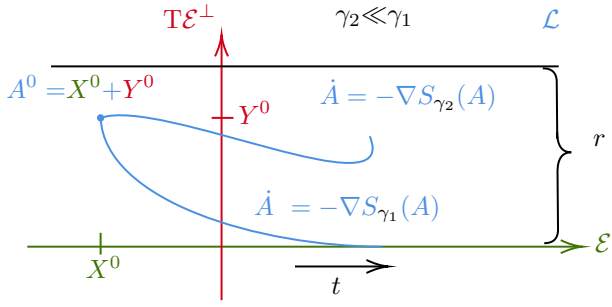
\begin{figure}
    \centering

\tikzset{every picture/.style={line width=0.75pt}} 

\begin{tikzpicture}[x=0.75pt,y=0.75pt,yscale=-1,xscale=1]

\draw [color={rgb, 255:red, 65; green, 117; blue, 5 }  ,draw opacity=1 ]   (190,155) -- (468,155) ;
\draw [shift={(470,155)}, rotate = 180] [color={rgb, 255:red, 65; green, 117; blue, 5 }  ,draw opacity=1 ][line width=0.75]    (10.93,-3.29) .. controls (6.95,-1.4) and (3.31,-0.3) .. (0,0) .. controls (3.31,0.3) and (6.95,1.4) .. (10.93,3.29)   ;
\draw [color={rgb, 255:red, 208; green, 2; blue, 27 }  ,draw opacity=1 ]   (288,182) -- (288,50) ;
\draw [shift={(288,48)}, rotate = 90] [color={rgb, 255:red, 208; green, 2; blue, 27 }  ,draw opacity=1 ][line width=0.75]    (10.93,-3.29) .. controls (6.95,-1.4) and (3.31,-0.3) .. (0,0) .. controls (3.31,0.3) and (6.95,1.4) .. (10.93,3.29)   ;
\draw [color={rgb, 255:red, 74; green, 144; blue, 226 }  ,draw opacity=1 ]   (227,90) .. controls (229,132) and (307,155) .. (369,155) ;
\draw    (202,64) -- (458,64) ;
\draw   (450.75,153.25) .. controls (455.42,153.21) and (457.73,150.86) .. (457.69,146.19) -- (457.44,117.42) .. controls (457.39,110.75) and (459.69,107.4) .. (464.36,107.36) .. controls (459.69,107.4) and (457.33,104.09) .. (457.27,97.42)(457.3,100.42) -- (457.05,71.94) .. controls (457.02,67.27) and (454.67,64.96) .. (450,65) ;
\draw [color={rgb, 255:red, 208; green, 2; blue, 27 }  ,draw opacity=1 ]   (283,90) -- (294,90) ;
\draw [color={rgb, 255:red, 65; green, 117; blue, 5 }  ,draw opacity=1 ]   (227,148) -- (227,159) ;
\draw [color={rgb, 255:red, 0; green, 0; blue, 0 }  ,draw opacity=1 ]   (325,165) -- (359.75,165) -- (370,165) ;
\draw [shift={(372,165)}, rotate = 180] [color={rgb, 255:red, 0; green, 0; blue, 0 }  ,draw opacity=1 ][line width=0.75]    (10.93,-3.29) .. controls (6.95,-1.4) and (3.31,-0.3) .. (0,0) .. controls (3.31,0.3) and (6.95,1.4) .. (10.93,3.29)   ;
\draw [color={rgb, 255:red, 74; green, 144; blue, 226 }  ,draw opacity=1 ]   (227,90) .. controls (257.75,79.25) and (378.75,138.25) .. (362.75,99.25) ;

\draw [color={rgb, 255:red, 74; green, 144; blue, 226 }] [fill={rgb, 255:red, 74; green, 144; blue, 226 }] (227,90) circle[radius= 0.1 em];

\draw (475,147) node [anchor=north west][inner sep=0.75pt]  [color={rgb, 255:red, 65; green, 117; blue, 5 }  ,opacity=1 ] [align=left] {$\displaystyle \mathcal{E}$};
\draw (257,31) node [anchor=north west][inner sep=0.75pt]  [color={rgb, 255:red, 208; green, 2; blue, 27 }  ,opacity=1 ] [align=left] {$\displaystyle \mathrm{T}\mathcal{E}^{\perp }$};
\draw (448,34) node [anchor=north west][inner sep=0.75pt]  [color={rgb, 255:red, 74; green, 144; blue, 226 }  ,opacity=1 ] [align=left] {$\displaystyle \mathcal{L}$};
\draw (474,97) node [anchor=north west][inner sep=0.75pt]   [align=left] {$\displaystyle r$};
\draw (295,80) node [anchor=north west][inner sep=0.75pt]  [color={rgb, 255:red, 208; green, 2; blue, 27 }  ,opacity=1 ] [align=left] {$\displaystyle Y^{0}$};
\draw (218,158) node [anchor=north west][inner sep=0.75pt]   [align=left] {$\displaystyle \textcolor[rgb]{0.25,0.46,0.02}{X^{0}}$};
\draw (303,125) node [anchor=north west][inner sep=0.75pt]   [align=left] {$\displaystyle \textcolor[rgb]{0.29,0.56,0.89}{\dot{A\ } =-\nabla S_{\textcolor[rgb]{0,0,0}{\gamma _{1}}}( A}\textcolor[rgb]{0.29,0.56,0.89}{)}$};
\draw (342,169) node [anchor=north west][inner sep=0.75pt]   [align=left] {$\displaystyle t$};
\draw (178,67) node [anchor=north west][inner sep=0.75pt]   [align=left] {$\displaystyle \textcolor[rgb]{0.29,0.56,0.89}{A^{0} =}\textcolor[rgb]{0.25,0.46,0.02}{X^{0}}\textcolor[rgb]{0.29,0.56,0.89}{+}\textcolor[rgb]{0.82,0.01,0.11}{Y^{0}}$};
\draw (336,69) node [anchor=north west][inner sep=0.75pt]   [align=left] {$\displaystyle \textcolor[rgb]{0.29,0.56,0.89}{\dot{A} =-\nabla S_{\textcolor[rgb]{0,0,0}{\gamma _{2}}}( A)}$};
\draw (344,33) node [anchor=north west][inner sep=0.75pt]   [align=left] {$\displaystyle  \gamma _{2}$$\displaystyle \ll $$\displaystyle \gamma _{1}$};

\end{tikzpicture}

    \caption{Regardless of how far from $\mathcal{E}$ we start training we can select the regularization parameter $\gamma$ large enough that the dynamics $\Dot{A}=-\nabla S_{\gamma}(A)$ converge to $\mathcal{E}$ exponentially fast. If a too low value of $\gamma$ is chosen, the dynamics might not converge to $\mathcal{E}$ at all.}
    \label{fig:attractor}
\end{figure}

\begin{theorem}[Equivariant subspace is attractor of regularized augmented gradient flow]\label{thm:attractor}
    Suppose that $R^{\mathrm{aug}}(A)\geq 0$ for every $A\in \mathcal{L}$ and let $S_{\gamma}(A)=S(A)\coloneqq R^{\mathrm{aug}}(A)+\frac{\gamma}{2}\lVert \Pi_{\mathcal{E}^{\perp}}A \rVert^2$ denote the regularized loss. Assume that $\ell$ is invariant, that the compatibility condition $\Pi_{\mathcal{L}}\Pi_G = \Pi_G \Pi_{\mathcal{L}}$ holds, and (i)-(iii)  above. Then, for any $r>0$, we can choose $\gamma>0$  large enough so that any curve started at a distance to $\mathcal{E}$ smaller than $r$ following the dynamics
    $$
    \Dot{A}=-\nabla S(A) = -\Pi_{\mathcal{L}}\nabla R^{\mathrm{aug}}(A) -\gamma\Pi_{\mathcal{E}^{\perp}}A
    $$
    will converge to $\mathcal{E}$ exponentially fast. 
\end{theorem}
\begin{proof}
    Let $X\in \mathcal{E}$ and $Y\in\mathrm{T}\mathcal{E}^{\perp}$ and write $A=X+Y$, so that $S(A)=R^{\mathrm{aug}}(A)+\frac{\gamma}{2}\lVert Y \rVert^2$. Due to (i), we then have
$$
\frac{\gamma}{2}\lVert Y \rVert^2\leq R^{\mathrm{aug}}(A) + \frac{\gamma}{2}\lVert Y \rVert^2\leq R^{\mathrm{aug}}(A^0) + \frac{\gamma}{2}\lVert Y^0 \rVert^2,
$$
where $A^0=X^0+Y^0$ is the starting point of the gradient flow $\Dot{A} = -\nabla S(A)$. Thus, we get the a priori bound
$$
\lVert Y \rVert^2\leq \frac{2}{\gamma}R^{\mathrm{aug}}(A^0)+\lVert Y^0 \rVert^2=:\alpha
$$
for every $A$. Now, a Taylor expansion of $\nabla R^{\mathrm{aug}}(A)$ around $X$ yields
\begin{align*}
\nabla R^{\mathrm{aug}}(A)=\Pi_{\mathcal{L}}\nabla R^{\mathrm{aug}}(X)+\Pi_{\mathcal{L}}(R^{\mathrm{aug}})''(X)Y&+O(\lVert Y \rVert^2).
\end{align*}
Note that the big-O is independent of $X$ due to assumption (iii). Due to Facts A and B above, we see that $\Pi_{\mathcal{L}}\nabla R^{\mathrm{aug}}(X)= \Pi_{\mathcal{E}}\nabla R(A) \in T\mathcal{E}$ and $\Pi_{\mathcal{L}}(R^{\mathrm{aug}})''(X)Y\in T\mathcal{E}^\perp$. Hence, as noted in to \cite[Prop. 3.14]{nordenfors2024optimizationdynamicsequivariantaugmented}, the dynamics $\dot{A}=-\nabla S(A)$ decouple in the following sense:
\begin{align}
    \Dot{X}&=-\Pi_{\mathcal{E}}\nabla R(A)+ O(\lVert Y \rVert^2), \nonumber \\
    \Dot{Y}&=-\Pi_{\mathcal{L}}(R^{\mathrm{aug}})''(X)Y-\gamma Y+O(\lVert Y \rVert^2).\label{eq:ydynamics}
\end{align}
Note that the term $-\gamma Y$ is due to the regularization. 
Differentiating the squared norm of $Y$, \eqref{eq:ydynamics} yields
\begin{align}\label{eq:grönwallcondition}
    &\frac{d}{dt}\Big(\frac{1}{2}\lVert Y\rVert^2\Big)=\langle \Dot{Y},Y \rangle\nonumber\\&=-\langle \Pi_{\mathcal{L}}(R^{\mathrm{aug}})''(X)Y,Y \rangle-\langle \gamma Y,Y \rangle+\langle O(\lVert Y\rVert^2),Y \rangle\nonumber\\
    &\leq -\sigma\lVert Y \rVert^2-\gamma \lVert Y\rVert^2+C\sqrt{\alpha}\lVert Y \rVert^2\nonumber\\&=(C\sqrt{\alpha}-\sigma-\gamma)\lVert Y \rVert^2,
\end{align}
where we used the a priori bound on $\Vert{Y}\Vert$ and assumption (ii).
From \eqref{eq:grönwallcondition} it follows by Grönwall's inequality that
\begin{align*}
\lVert Y\rVert^2 \leq \lVert Y^0\rVert^2 \mathrm{exp}\Big(2(C\sqrt{\alpha}-\sigma-\gamma)t \Big).
\end{align*}
Hence, if $\gamma > C\sqrt{\alpha}-\sigma$, $\lVert Y\rVert^2$ decays exponentially to 0. Since $\alpha$ only depends on $\Vert{Y^0}\Vert$, we see that we can achieve that by choosing $\gamma$ in dependence of $r$, which is what was to be proved.
\end{proof}
\begin{remark}
    From an intuitive point of view, heavily regularizing the non-equivariant part $\Pi_{\mathcal{E}^\perp} A$ of the linear layers will also force the dynamics down to $\mathcal{E}$ in the non-augmented case. This is to some degree true, but the augmented case is different from the non-augmented one in two crucial ways. To see this, let us, in the same way as in  the proof of Theorem \ref{thm:attractor} Taylor expand the gradient of the nominal risk $\nabla R(A)$ around an $X\in \mathcal{E}$ and write down the $Y$ dynamics as
    \begin{align*}
        \Dot{Y} = -\Pi_\mathcal{L}\nabla R(X) - \Pi_\mathcal{L}R''(X)Y&-\gamma Y+O(\lVert Y\rVert^2),
    \end{align*}
    which yields
    \begin{align*}
        \frac{d}{dt}\Big(\frac{1}{2}\lVert Y \rVert^2\Big) &=-\langle\Pi_\mathcal{L}\nabla R(X),Y\rangle -\langle \Pi_{\mathcal{L}}R''(X)Y,Y \rangle\\&\quad\,-\langle \gamma Y,Y \rangle+\langle O(\lVert Y\rVert^2),Y \rangle.
    \end{align*}
    Now, in the non-augmented case we are not guaranteed that $\langle{\mathcal{L}} \nabla R(X),Y\rangle=0$, as opposed to the augmented case where Fact A ensures that the inner product vanishes.
    
    Assuming $\langle \Pi_{\mathcal{L}}\nabla R(X),Y\rangle=0$, an argument similar to that in the proof of Theorem \ref{thm:attractor} would yield a similar result. However, in this case, the parameter $\sigma$ would be different and would be a smaller number (which would lead to a higher required value of $\gamma$). To see this, let us use the following formula, which follows by \cite[Lemma 3.11]{nordenfors2024optimizationdynamicsequivariantaugmented} 
    \begin{align*}
    \langle(R^{\mathrm{aug}})''(A)Y,Y\rangle&=\int_G\langle R''(A)\rho(g)Y, \rho(g)Y\rangle d\mu(g) 
    \end{align*}
    Now, if $Y\in T\mathcal{E}^\perp$, it is not hard to see that $\rho(g)Y\in T\mathcal{E}^\perp$ also for all $g$. Hence, if $\langle R''(A)Y,Y\rangle \geq \overline{\sigma}\Vert Y \Vert^2$ for all $Y\in T\mathcal{E}^\perp, A\in \mathcal{E}$, the above integral is bounded below by $\overline{\sigma}\Vert Y \Vert^2 $. In this sense, $\sigma \geq \overline{\sigma}$, and therefore we can potentially get away with a lower value for $\gamma$ when regularizing the augmented risk compared to the nominal risk.
\end{remark}
\begin{remark}
    We can also say something about what happens if we begin training in a neighborhood in $\mathcal{L}$ of a strict local minimum $X^*$ for the equivariant mode training. Let us write $A=(X,Y)$, where $X\in \mathcal{E}$ and $Y\in \mathrm{T}\mathcal{E}^\perp$. Consider the dynamics
    \begin{align*}
        \Dot{X} &= -\Pi_{\mathcal{E}}\nabla S(X,Y)=-\Pi_{\mathcal{E}}\Pi_{\mathcal{L}}\nabla R^{\mathrm{aug}}(X,Y),\\
        \Dot{Y} &= -\Pi_{\mathcal{E^\perp}}\nabla S(X,Y) = -\Pi_{\mathcal{E}^{\perp}}\Pi_{\mathcal{L}}\nabla R^{\mathrm{aug}}(X,Y) - \gamma Y.
    \end{align*}
    Let $\Delta X=X-X^*$. Taylor expanding around $(X^*,0)$ gives
    \begin{align*}
    \Dot{\begin{bmatrix}
        \Delta X\\
        Y
    \end{bmatrix}} &= -\Big((R^{\mathrm{aug}})''(X^*,0)+\begin{bmatrix}
        0&0\\
        0&\gamma I
    \end{bmatrix}\Big)\begin{bmatrix}
        \Delta X\\
        Y
    \end{bmatrix}\\&\quad\, + O(\lVert \Delta X \rVert^2+\lVert Y \rVert^2)
\end{align*}
By choosing $\gamma$ large enough, we can ensure that the second derivative is positive definite. Thus, there is a neighborhood $U$ of the point $(X^*,0)$ such that starting the training at a point $(X^0,Y^0)\in U$, we guarantee that the regularized augmented training will converge to $(X^*,0)$.
\end{remark}

\section{Experimenting with Different Values of $\gamma$}
We perform a numerical experiment to confirm the practical relevance of our result. In the experiment we use SGD with random augmentations, which is different from what we do in our theoretical development, where we assume training with perfect augmentation and gradient flow, but also closer to what is done in practice. Since our results here in particular are about stability, they should also have an impact in the stochastic setting.
\paragraph{Experiment description} We consider CNN:s with two convolutional layers of 16 channels each, with filters of size $3\times3$, followed by a fully connected layer (both without biases). We use $\tanh$ activation functions and a cross-entropy loss. Since the activation function acts pixel-wise, they are equivariant. The networks are trained in three configurations: (i) in equivariant mode, (ii) on (randomly) augmented data (iii) on non-augmented data. The group action we consider is the discrete group of 90 degree rotations, and aim to build invariant networks. As noted above, these networks obey the compatibility condition. Since the group acts trivially on the output space,  the invariance of the loss function $\ell$ is trivial.

We initialize the equivariant networks with Gaussian weights. We then copy those weights to the non-equivariant networks, and subsequently perturb them with Gaussian noise (independently from each other). We train the networks on augmented and non-augmented data, respectively, using SGD on MNIST for 10 epochs with a learning rate of 1e-3 and a batch size of 10. Both non-equivariant networks are regularized, using four different values of the regularization constant $\gamma$, 1e-4,1e-2,1e0 and 1e2. We record after each gradient step the distance of the non-equivariant models to $\mathcal{E}$. 
 The experiment is repeated 30 times for each value of $\gamma$. All code used in the experiment is made available at \url{https://github.com/usinedepain/eq_aug_reg_release}

\paragraph{Results} In Figure \ref{fig:proj}, we plot the evolution of the distance to $\mathcal{E}$ along the training -- the opaque lines are the median values. The graphs look just as one expects -- for high values of $\gamma$, both non-equivariant models stay close to $\mathcal{E}$. However, already for $\gamma=\mathrm{1e0}$, the model trained on non-augmented data drifts considerably, while the augmented model stays close to $\mathcal{E}$. We also see that when the constant is chosen too low ($\gamma=\textrm{1e-4}$), the augmented model will also drift from $\mathcal{E}$.


\begin{figure}
\includegraphics[width=.24\textwidth]{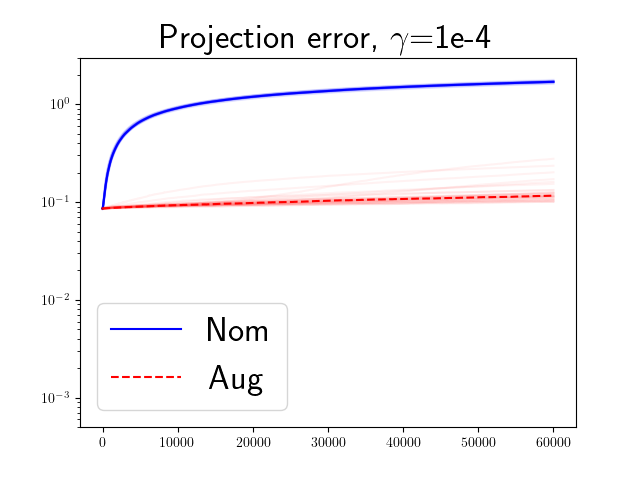}\includegraphics[width=.24\textwidth]{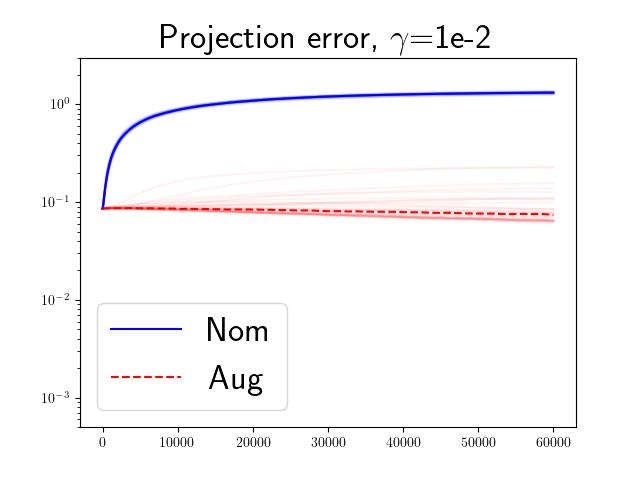}\\\includegraphics[width=.24\textwidth]{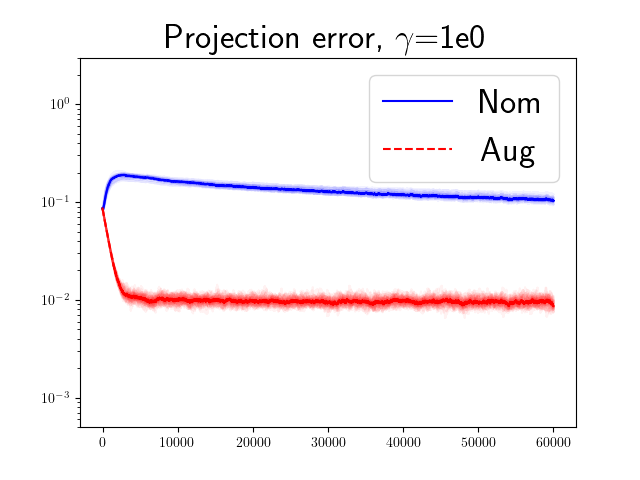}\includegraphics[width=.24\textwidth]{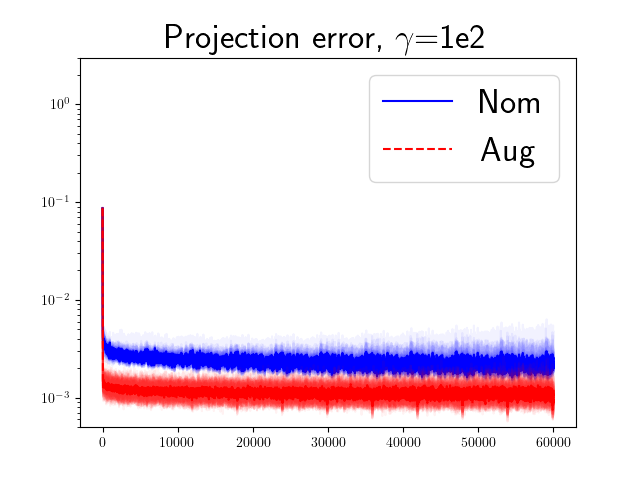}
\caption{Projection errors for the two non-equivariant models for different values of $\gamma$. Notice the logarithmic $y$-scale. Opaque lines are medians, and transparent lines are individual runs. Best viewed in color. \label{fig:proj}}
\end{figure}
\section{Conclusion}
We investigated the relationship between manifestly equivariant neural network architectures with data augmentation.  Using the 'softer' approach of data augmentation and a regularization term to punish distance to an equivariant subspace $\mathcal{E}$, we obtain equivariant models without having to restrict the layers of the network to be equivariant a priori. The theoretical results rely heavily on our previous paper \cite{nordenfors2024optimizationdynamicsequivariantaugmented}. We have also seen, with a small numerical experiment, that these results are born out in practice, even when using SGD with random augmentations.

For future work, it would be interesting to do larger experiments to see if these methods can give lower test error than manifest equivariance. In particular, it would be interesting to see whether a fine-tuning of $\gamma$ could make the networks avoid bad local minima on $\mathcal{E}$ while still being attracted to good ones.
\section{Acknowledgments}
This work was partially supported by the Wallenberg AI, Autonomous Systems and Software Program (WASP) funded by the Knut and Alice Wallenberg Foundation. 
\bibliography{biblio}
\bibliographystyle{IEEEtran}

\end{document}